\documentclass[a4paper,journal]{IEEEtran}
\usepackage{generic}
\usepackage{cite}
\usepackage{amsmath,amssymb,amsfonts}
\usepackage{algorithmic}
\usepackage{graphicx}
\usepackage{textcomp}
\usepackage{amsmath,amsfonts}
\usepackage{amsthm}
\usepackage{algorithmic}
\usepackage{algorithm}
\usepackage{array}
\usepackage[caption=false,font=normalsize,labelfont=sf,textfont=sf]{subfig}
\usepackage{textcomp}
\usepackage{hyperref}
\usepackage{stfloats}
\usepackage{url}
\usepackage{verbatim}
\usepackage{graphicx}
\usepackage{cite}
\usepackage{color}   
\usepackage{nomencl}
\makenomenclature
\newtheorem{theorem}{Theorem}[section]
\newtheorem{corollary}{Corollary}[theorem]
\newtheorem{lemma}[theorem]{Lemma}

\def\BibTeX{{\rm B\kern-.05em{\sc i\kern-.025em b}\kern-.08em
    T\kern-.1667em\lower.7ex\hbox{E}\kern-.125emX}}
\begin{document}
\title{Online 3D Bin Packing with Fast Stability Validation and Stable Rearrangement Planning}
\author{
        Ziyan Gao${}^1$, Lijun Wang${}^2$, Yuntao Kong${}^2$, Nak Young Chong${}^*$
\thanks{This work was supported in part by , and in part by .}
\thanks{Manuscript received ; revised .}}

\maketitle

\begin{abstract}
The Online Bin Packing Problem (OBPP) is a sequential decision-making task in which each item must be placed immediately upon arrival, with no knowledge of future arrivals. Although recent deep-reinforcement-learning methods achieve superior volume utilization compared with classical heuristics, the learned policies cannot ensure the structural stability of the bin and lack mechanisms for safely reconfiguring the bin when a new item cannot be placed directly.
In this work, we propose a novel framework that integrates packing policy with structural stability validation and heuristic planning to overcome these limitations. Specifically, we introduce the concept of Load Bearable Convex Polygon (LBCP), which provides a computationally efficient way to identify stable loading positions that guarantee no bin collapse. Additionally, we present Stable Rearrangement Planning (SRP), a module that rearranges existing items to accommodate new ones while maintaining overall stability.
Extensive experiments on standard OBPP benchmarks demonstrate the efficiency and generalizability of our LBCP-based stability validation, as well as the superiority of SRP in finding the effort-saving rearrangement plans. Our method offers a robust and practical solution for automated packing in real-world industrial and logistics applications. A demonstration video of the real-world deployment is available online.\footnote{\url{https://drive.google.com/file/d/1Uhnk7P-2geSftrlKKPW9NTe0NMLXweK1/view?usp=drive_link}}
\end{abstract}


\section{Introduction}
\IEEEPARstart{T}{he} 3D Online Bin Packing Problem (OBPP) plays a pivotal role in a variety of industrial and logistical applications, from automated warehouses to distribution centers. The online setting requires real-time decisions about each incoming item’s placement with limited or no knowledge of future arrivals. This complexity makes traditional heuristic-based approaches insufficient for optimal packing~\cite{HEU_christensen2016multidimensional}. Consequently, deep reinforcement learning (DRL) approaches have gained traction as a promising alternative for optimizing bin utilization.

Recent research~\cite{
DRL_Dataset_zhao2021online, 
DRL_PCT,
DRL_2021_yang2021packerbot,
Stability_zhao2022learning,
DRL_GOPT, 
DRL_wu2024efficient,
DRL_kang2024gradual,
DRL_mu20253d,
DRL_Precedence_hu2020tap,
DRL_MP_Precedence_xu2023neural, 
DRL_unpack_song2023towards,  
DRL_MP_yang2023heuristics,
Offline_DRL_zhang2021attend2pack,
Offline_DRL_zhu2021learning,
Offline_DRL_jiang2021learning,
Offline_DRL_que2023solving,
Offline_DRL_pan2024ppn,
Offline_DRL_packing_order} has demonstrated that DRL models outperform the heuristic methods in maximizing bin utilization. While maximizing bin utilization is crucial, another critical aspect of OBPP is ensuring bin stability during and after each packing operation. In industrial applications, unstable packing configurations can lead to safety risks, inefficient space usage, and logistical challenges. Despite the importance of stability, a robust yet computationally efficient solution remains an open question in 3D OBPP. Recent works have started exploring stability constraints beyond simple geometric fitting. Some approaches leverage physics-based modeling, where physical constraints such as friction, weight distribution are explicitly formulated to assess stability\cite{Stability_Wang_8794049, Stability_10556686}. While these methods show promise, they often suffer from high computational costs, limiting their practicality for large-scale real-time applications. Additionally, many of these approaches are tailored to specific item types or packing rules, reducing their generalizability.

We propose a computationally efficient, stability-ensured bin packing framework. Our main innovation is to validate and maintain stable load configurations through a unique concept called Load-Bearable Convex Polygon (LBCP). LBCPs provide a lightweight yet powerful representation that captures the set of load-bearing regions in a packed bin without explicit knowledge of each item’s mass or precise center of gravity (CoG). Importantly, the time complexity 
of the proposed stability validation process is nearly constant and is readily integrated into the DRL training framework, ensuring only stable placements are proposed for each incoming item. 

Beyond guaranteeing stability, we handle the common situation in which the bin must be reconfigured to fit an incoming item that cannot be placed directly. Unlike the previous work~\cite{DRL_unpack_song2023towards, DRL_MP_yang2023heuristics}, we introduce the Stable Rearrangement Planning (SRP) module to devise a series of efficient and stability-ensured operations to either make space for a new item or enhance overall bin utilization. Specifically, the proposed SRP is a layered approach combining Monte Carlo Tree Search~\cite{MCTS} that finds the feasible operational sequence given the limited search steps, with a post-processing $A^*$ search~\cite{A_star} that shortens the overall sequence under the heuristic guidance.

Extensive experiments on the widely used RS Dataset~\cite{DRL_Dataset_zhao2021online} indicate that our approach not only matches the packing efficiency of advanced DRL methods, but also ensures the stability of each newly placed item. Compared with a baseline method, our LBCP-based method significantly reduces the computational overhead, enabling real-time deployment. Additionally, we demonstrate that the SRP technique significantly enhances bin utilization by effectively reorganizing a small number of items to make room for new entries. In simulated trials, our SRP framework outperforms simple unpack-and-repack heuristics~\cite{DRL_MP_yang2023heuristics} in both final utilization and the number of operations. Finally, we showcase a robotic bin packing system for cuboid objects.

In summary, the main contributions of this paper include:
\begin{itemize}
    \item A novel stability validation mechanism for OBPP based on LBCPs, which efficiently checks stability without requiring prior knowledge of accurate item mass distributions.
    \item An integrated DRL pipeline that leverages LBCPs as an action-masking mechanism, enabling stable and high utilization bin packing decisions in real time.
    \item The SRP framework combining tree search and sequence refinement to free up space for new items or reorganize the bin occupancy in an efficient manner.
    \item A thorough empirical evaluation on standard bin packing benchmarks, showing both significantly improved utilization and guaranteed physical stability, along with promising real-world feasibility.
\end{itemize}

\section{Related Work}
\subsection{DRL Model for Bin Packing}

In recent years, DRL has become a powerful tool for solving both online~\cite{DRL_Dataset_zhao2021online, 
DRL_PCT,
DRL_2021_yang2021packerbot,
Stability_zhao2022learning,
DRL_GOPT, 
DRL_wu2024efficient,
DRL_kang2024gradual,
DRL_mu20253d,
DRL_Precedence_hu2020tap,
DRL_MP_Precedence_xu2023neural, 
DRL_unpack_song2023towards,  
DRL_MP_yang2023heuristics} and offline~\cite{Offline_DRL_zhang2021attend2pack,
Offline_DRL_zhu2021learning,
Offline_DRL_jiang2021learning,
Offline_DRL_que2023solving,
Offline_DRL_pan2024ppn,
Offline_DRL_packing_order} BPPs. Offline DRL emphasizes efficiency for large-scale problems, while online DRL focuses on maximizing space usage under feasibility constraints, with varying representations of bin and item states.

For offline BPP, many works aim to accelerate search via hierarchical structures and pruning. Zhu~{\it et al.}~\cite{Offline_DRL_zhu2021learning} use Convolutional Neural Networks (CNNs)~\cite{CNN} to skip poor candidates, while Zhang~{\it et al.}~\cite{Offline_DRL_zhang2021attend2pack} and Jiang~{\it et al.}~\cite{Offline_DRL_jiang2021learning} apply two-stage decompositions. Pan~{\it et al.}~\cite{Offline_DRL_pan2024ppn} and Wang~{\it et al.}~\cite{Offline_DRL_packing_order} introduce learned proposal modules to generate compact placements. Que~{\it et al.}~\cite{Offline_DRL_que2023solving} adopt transformers to guide search toward promising actions. These methods aim to reduce the cost of exhaustive exploration while preserving strong packing quality.

In online BPP, early work~\cite{DRL_Dataset_zhao2021online, Stability_zhao2022learning} discretizes the bin into grids for efficient collision checks. Yang~{\it et al.}~\cite{DRL_MP_yang2023heuristics} use 3D voxels to encode bin states, though limited resolution hinders generalization. To address this, Zhao~{\it et al.}~\cite{DRL_PCT} propose configuration-node trees for continuous placement, and Xiong~{\it et al.}~\cite{DRL_GOPT} constrain actions using Empty Maximal Spaces (EMS)~\cite{HEU_EMS}. CNNs are widely used as backbones~\cite{DRL_Dataset_zhao2021online, Stability_zhao2022learning, DRL_MP_yang2023heuristics, DRL_2021_yang2021packerbot, DRL_unpack_song2023towards, DRL_wu2024efficient}, while transformers~\cite{Attention} enhance item-subspace encoding~\cite{DRL_GOPT, DRL_kang2024gradual, DRL_mu20253d}.

Despite efficiency improvements, stack stability remains underexplored. Rather than proposing a new DRL model, this study introduces a general framework with a fast stability validation module to support policy learning in OBPP.

\subsection{Stability Validation}
Traditional stability validation methods fall into two categories: geometry-based and static equilibrium-based. Geometry-based methods, such as those by Ramos {\it et al.}~\cite{Stability_ramos2016physical}, assess stability via the item's CoG and support polygon, while Gzara {\it et al.}~\cite{Stability_gzara2020pallet} and Zhu {\it et al.}~\cite{Stability_ZHU2024109814} impose contact area constraints. Ali {\it et al.}~\cite{Stability_ali2025static} compared several geometric criteria and their impact on packing quality and real-world stability. These approaches can be overly conservative or too relaxed, failing to balance bin utilization and stability. On the static equilibrium side, Hauser~\cite{Stability_Wang_8794049} formulated the problem as convex optimization under known geometry and friction, while Liu {\it et al.}~\cite{Stability_10556686} introduced force-balance equations to analyze weak points in stacked structures. Despite their accuracy, such methods are computationally expensive and lack scalability for OBPP.

Several works have integrated stability modules into DRL frameworks~\cite{DRL_Dataset_zhao2021online, Stability_zhao2022learning, Stability_peiwen, DRL_MP_yang2023heuristics, Stability_10631682, Stability_zhang2025physics}. Zhao {\it et al.}~\cite{DRL_Dataset_zhao2021online} used geometric rules~\cite{Stability_ramos2016physical} to filter unstable placements and later proposed a stacking tree to recursively validate item stability~\cite{Stability_zhao2022learning}, assuming known and uniform mass—an unrealistic constraint in practice. Zhou {\it et al.}~\cite{Stability_peiwen} introduced an ``empty map'' aligned with the heightmap to infer support regions, showing improved results over the baseline but at the cost of conservative behavior. Yang {\it et al.}~\cite{DRL_MP_yang2023heuristics} enhanced stability checks by combining geometry-based rules with neural prediction. Wu {\it et al.}~\cite{Stability_10631682} and Zhang {\it et al.}~\cite{Stability_zhang2025physics} applied iterative action masking to improve RL-based palletization but faced generalization challenges under distribution shifts. Beyond learning-based strategies, Mazur {\it et al.}~\cite{Stability_MAZUR2025100329} employed real-time physics simulation for cargo loading, revealing discrepancies between static models and dynamic behavior.

\subsection{Planning for Item Rearrangement}
Some recent work explored the unpacking mechanism~\cite{DRL_unpack_song2023towards, DRL_MP_yang2023heuristics}. Song {\it et al.}~\cite{DRL_unpack_song2023towards} proposed a packing-and-unpacking network (PUN) that decides either to pack the item or to remove a previously placed item from the bin to free up space, and a DRL model was employed to learn how to switch between packing and unpacking for better space utilization. While their experimental findings demonstrated an enhancement in bin utilization, their model might face a generalization challenge. Additionally, the operational sequence might prolong due to the observe-then-execute approach. Yang {\it et al.}~\cite{DRL_MP_yang2023heuristics} proposed a straightforward unpacking strategy that takes into account both item size and wasted space. Nevertheless, this efficiency-oriented heuristic disregards stability constraints, frequently leading to local optima. Hu {\it et al.}~\cite{DRL_Precedence_hu2020tap} and Xu {\it et al.}~\cite{DRL_MP_Precedence_xu2023neural} introduced a different scenario, specifically targeting the transport-and-pack (TAP) problem. In particular, they applied a precedence graph to represent the constraints of the packing order. In our study, we utilized the precedence graph to streamline the operational sequence in order to minimize operational time.
\begin{figure*}
    \centering
    \includegraphics[width=0.8\textwidth]{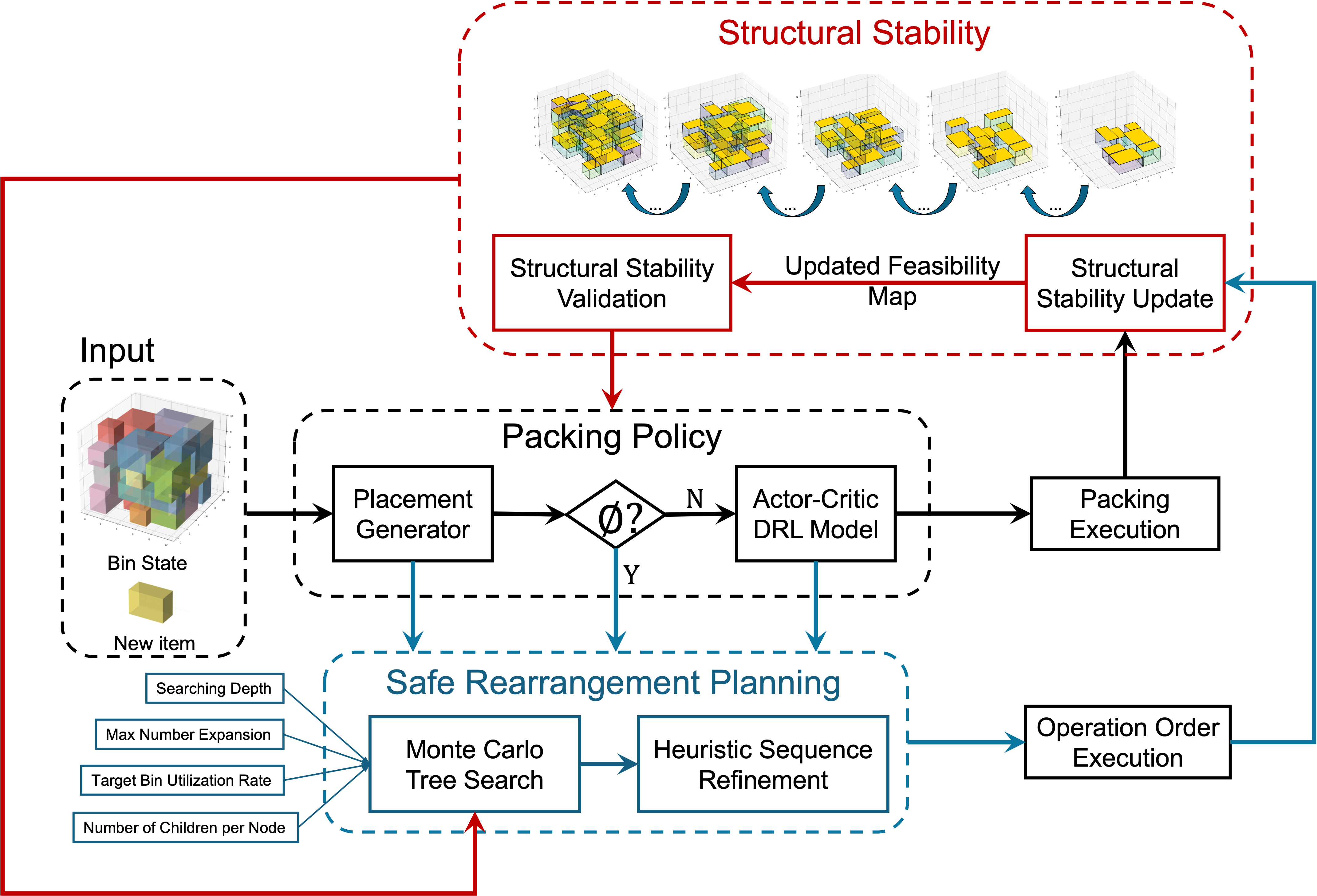}
    \caption{The proposed online bin packing framework consists of three key modules: the structural stability module (Sections~\ref{section:lbcp} and~\ref{section:ssv}), the packing policy module (Section~\ref{section: stability ensured drl}), and the safe rearrangement planning module (Section~\ref{section:srp}). Stability is ensured by the structural stability validation process, which filters out unstable placements based on an updated feasibility map. This validation is seamlessly integrated into both the packing policy and rearrangement planning modules to enforce stability constraints throughout. The packing policy module selects a loading position for each new item. If no stable placement is found, the safe rearrangement planning module generates a sequence of rearrangement actions on already packed items to create feasible space for the incoming item.}
    \label{fig: framework}
\end{figure*}
\section{Method}
\subsection{Problem Statement and Framework Overview}
The bin and items are cuboidal. We adopt an online packing scheme where items arrive sequentially, and only the latest item is accessible. Based on the current bin state and newly observed item, the robot is requested to determine the loading position of this item. 
Inspired by prior works~\cite{DRL_unpack_song2023towards, DRL_MP_yang2023heuristics}, showing that unpacking can improve bin utilization, this work allows the robot to temporarily move previously packed items to a staging buffer outside the bin to make space for the new item, and later place them back. We will detail this procedure in safe rearrangement planning section ~\ref{section:srp}. The process terminates when no stable placement can be found for the incoming item.

As for the object physical property, we assume the weight of each item is unknown, but the center of gravity of each item is subject to a bounded uncertainty, {\it i.e.}, the shift of CoG, denoted by $\mathcal{C}_i$, along the axes of its body frame is proportional to the respective dimensions. We use $\mathcal{O}_i=[w_i,d_i, h_i]^\top$ to denote the dimension of the $i$th item. If the highest ratio is given as $\delta_{CoG}$, then, the uncertainty can be quantified by Eq.~\ref{eq:CoG set}.

\begin{equation}
\mathcal{C}_i = \{\mathbf{g}_i + 
\begin{bmatrix}
\delta_w{w_i} \\
\delta_d{d_i} \\
\delta_h{h_i}
\end{bmatrix}, |\delta_w|,|\delta_d|,|\delta_h|\leq \delta_{CoG}\}
\label{eq:CoG set}
\end{equation}

We represent the loading position of $i$th item as $\mathcal{L}_i = (x_i, y_i, z_i)$. Let $\mathcal{I}_i = (\mathcal{O}_i, \mathcal{L}_i)$ denote the state of the $i$th item inside the bin. The objective is to identify stable loading positions and then select the one that facilitates the future accommodation for upcoming items based on $\{\mathcal{I}_1,\mathcal{I}_2,...,\mathcal{I}_{i-1}, \mathcal{O}_{i}\}$. This problem involves sequential decision making under uncertainty, balancing optimal space utilization with the rigorous constraints imposed by load stability.

\subsection{Overview of the Proposed Framework}

The proposed framework for \textit{online bin packing} consists of three integral components, as illustrated in Fig.~\ref{fig: framework}:

\begin{enumerate}
    \item A \textbf{packing policy module} responsible for proposing feasible placement options for incoming items
    \item A \textbf{structural stability module} that ensures the mechanical integrity of the packed bin
    \item A \textbf{safe rearrangement planning module} that generates feasible rearrangement strategies when direct placement is not possible.
\end{enumerate}

At each decision step, the system takes as input the current state of the bin and a newly arriving item. A \textit{placement generator}, based on heuristics, first proposes a set of candidate loading positions by analyzing the current bin configuration. These candidate positions are then passed to the \textit{structural stability module}, which evaluates whether placing the item at each position would compromise the stability of the bin. Any placements that could destabilize the bin structure are discarded.

If one or more stable placements are available, the system uses an \textit{Actor-Critic DRL model} to select the most suitable option from the candidate set. If no stable placement exists, the system activates the \textit{safe rearrangement planning module}. This module leverages a combination of \textit{Monte Carlo Tree Search (MCTS)} and \textit{heuristic sequence refinement} to generate a rearrangement plan. The goal is to create enough space for the new item while ensuring that the stability of the bin configuration is preserved throughout the entire process.

The remainder of this section is organized as follows.
First, we describe the \emph{structural stability module}, which verifies that every placement satisfies static-equilibrium constraints. Next, we show how this module is integrated with the learning-based packing policy to filter unstable actions online. Finally, we present the safe rearrangement planning module.

\subsection{Structural Stability}
\subsubsection{Static Equilibrium}
An item is statically stable when its center of gravity (CoG) lies within its \emph{support polygon}, defined as the minimal convex hull of all contact points \cite{Stability_ramos2016physical}. Because we assume the support polygon fully covers the CoG uncertainty, this condition applies directly. Fig.~\ref{fig:support polygon} illustrates a stable placement of item~$a$: the convex hull of its contact regions (darker-shaded overlapping area) completely encloses the CoG uncertainty (yellow).

\begin{figure}
    \centering
    \includegraphics[width=0.3\textwidth]{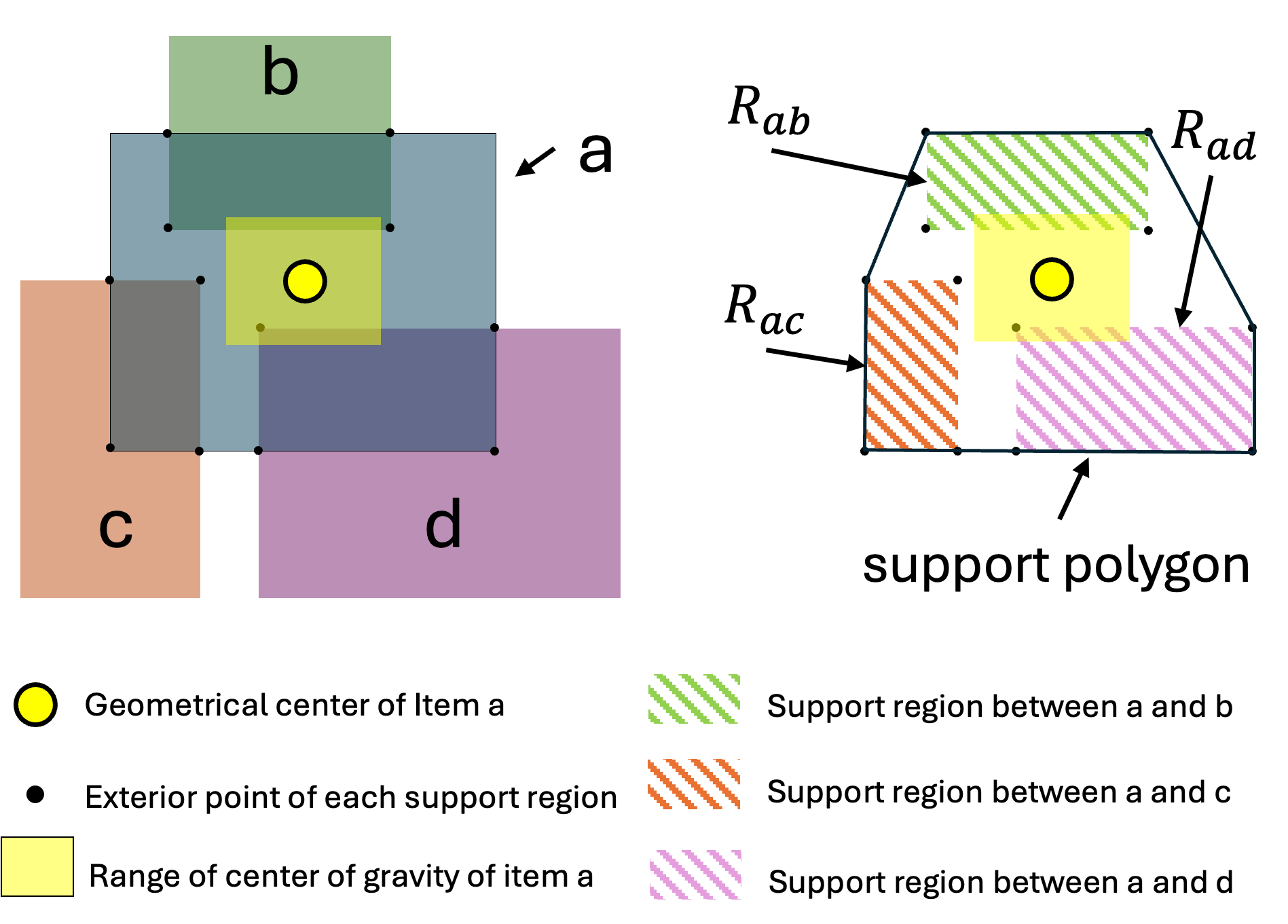}
    \caption{Support regions and support polygon. Item $a$ is considered stable if its CoG is inside the support polygon.}
    \label{fig:support polygon}
\end{figure}

\begin{figure}
    \centering
    \includegraphics[width=0.35\textwidth]{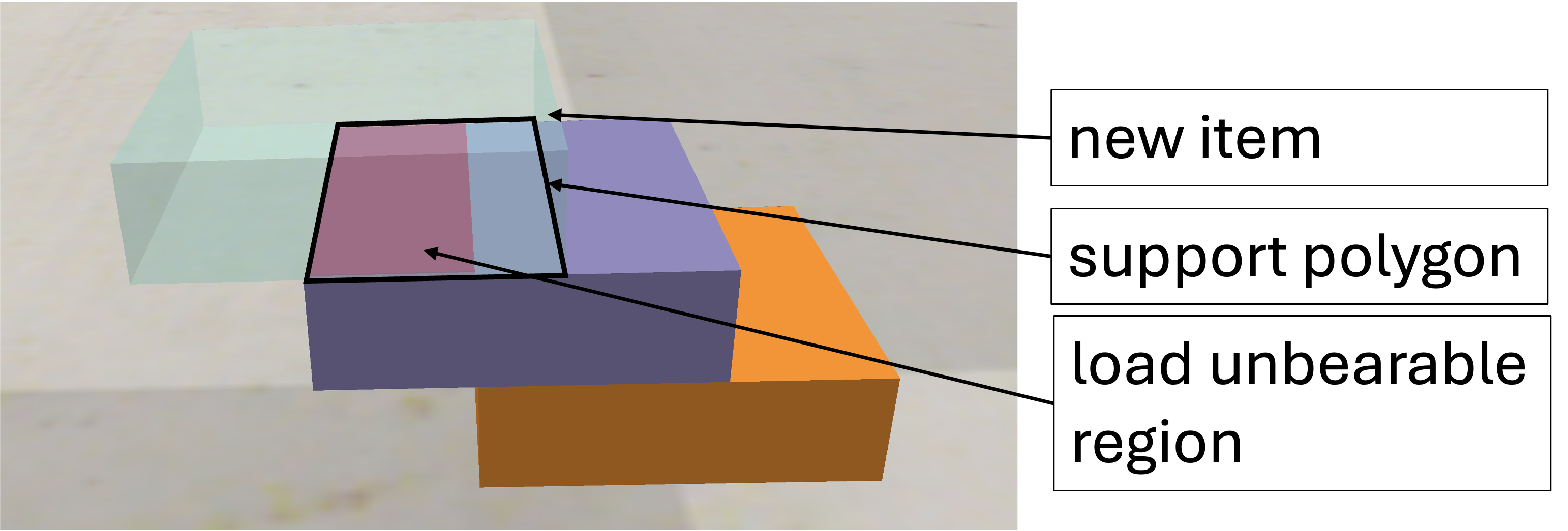}
    \caption{Load unbearable region in the support polygon.}
    \label{fig:unbearable}
\end{figure}

Specifically, the support polygon of a new item $\mathcal{I}_{new}$ can be calculated in two steps. The first step is to find all items (including ground) that support $\mathcal{I}_{new}$ and identify all support regions. The second step is to calculate the convex hull of support regions. Let $\mathcal{B}_t = \{\mathcal{I}_1,\mathcal{I}_2,...,\mathcal{I}_i,...,\mathcal{I}_m\}$ denote the current bin state. We define $\mathbf{CH}(\{...\})$ as the function that takes a set of support regions to compute their extreme points. Then, all support regions of $\mathcal{I}_{new}$ can be represented by $\{\mathcal{I}_{new}\cap\mathcal{I}_j|z_j+h_j=z_{new}, \mathcal{I}_j\in\mathcal{B}_t\}$, and 
the support polygon can be specified by $\mathbf{CH}(\{\mathcal{I}_{new}\cap\mathcal{I}_j|z_j+h_j=z_{new}, \mathcal{I}_j\in\mathcal{B}_t\})$. We use $\mathbf{P}_{new}^\triangle$, where the superscript $\triangle$ infers the convexity of the polygon, to represent the support polygon associated with $\mathcal{I}_{new}$. 


The implicit assumption of this stability is that the points on hull vertices can produce any support forces for the item on top of it, which can be readily compromised in scenarios involving stacks of more than two layers. Fig.~\ref{fig:unbearable} illustrates the case that the support polygon may fail to bear the gravitational forces by the new item, since there is no structural support beneath the item in the middle. Thus, the geometric intersection alone cannot guarantee true load‐bearing capacity, and the item can topple despite appearing stable in a purely geometric analysis.

\subsubsection{Load Bearable Convex Polygon}
\label{section:lbcp}
Load Bearable Convex Polygon (LBCP) is a convex polygon parallel to the horizontal plane and is characterized by its ability to support any gravitational forces at any point within it. Each LBCP is usually located at the top face of the packed item, and we use $(\mathbf{P}_i^\triangle, h^s_{i})$ to denote a specific LBCP where $h^s_i = z_i + h_i$, while a special LBCP is located at the bottom of the bin and is represented by $(\mathbf{P}_0^\triangle, 0)$.

The number of LBCPs increases with the growing number of packed items, and we use $\mathcal{P}_t$ to represent the set of LBCPs in the bin state $\mathcal{B}_t$. Excluding $(\mathbf{P}_0^\triangle, 0)$, the number of LBCPs is equal to the number of packed items.

\begin{lemma}
If a cuboidal item is placed stably at the bottom of the bin, then its entire top face is an LBCP.\label{lemma:1}
\end{lemma}
\begin{proof}
An item is considered statically stable if its CoG lies within the support polygon. For an item resting at the bottom of the bin, the support polygon is simply the item’s base in full contact with the bottom plane. 
Any load placed on the item’s top face is transmitted vertically downward through the item to the bottom. The bottom, in turn, provides equal and opposite reaction forces that prevent tipping or toppling. In other words, placing a load on any point of the top face will not shift the item’s CoG outside its support polygon. Thus, the top face can fully bear that load without overturning. Hence, the item’s top face is a valid LBCP.
\end{proof}

Our stability validation is based on the inclusiveness relationship between the CoG of the new item and its support polygon. The substantial difference from \cite{Stability_ramos2016physical} is that the support polygon is calculated using LBCPs, given by Eq.~\ref{eq:LBCP}. 
\begin{equation}
\mathbf{P}_{new}^\triangle = \mathbf{CH}(\{\mathcal{I}_{new}\cap\mathbf{P}^\triangle_i|z_{new}=h^s_i, \mathbf{P}^\triangle_i\in \mathcal{P}_t\})
    \label{eq:LBCP}
\end{equation}  
where $\mathcal{P}_t = \{\mathbf{P}^\triangle_0, \mathbf{P}^\triangle_1,\mathbf{P}^\triangle_2,...,\mathbf{P}^\triangle_i,...,\mathbf{P}^\triangle_m\}$.\\

\begin{theorem}
The support polygon calculated based on the current LBCPs is also an LBCP.\label{theory:1}        
\end{theorem}
\begin{proof}
The support polygon is calculated based on all intersected regions between $\mathcal{I}_{new}$ and LBCPs. From the definition of LBCPs, any point inside the intersected regions can bear any magnitude of gravitational forces. The operation $\mathbf{CH}(\{\dots\})$ guarantees convexity.
\end{proof}

\begin{corollary}
If the CoG of a newly packed item $\mathcal{I}_{new}$ falls in its associated support polygon that is calculated based on LBCPs, the stability of $\mathcal{I}_{new}$ can be ensured.\label{corollary:1}
\end{corollary}
\begin{proof}
According to Theorem~\ref{theory:1}, the support polygon is considered an LBCP. From the LBCP definition, any point inside the support polygon can counteract any gravitational force. Thus, if the CoG of the newly added item $\mathcal{I}_{new}$ is located within the support polygon, the gravitational force will be balanced.
\end{proof}

\begin{corollary}
If $\mathcal{I}_{new}$ satisfies Corollary~\ref{corollary:1}, the stability of items that support $\mathcal{I}_{new}$ will remain unaffected in the presence of $\mathcal{I}_{new}$.\label{corollary:2}
\end{corollary}
\begin{proof}
Suppose $\{\mathcal{I}_1,...,\mathcal{I}_i,...\mathcal{I}_g\}$ that support $\mathcal{I}_{new}$. The loading forces from $\mathcal{I}_{new}$ to $\{\mathcal{I}_1,...,\mathcal{I}_i,...,\mathcal{I}_g\}$ are $\{\mathbf{r}_1,...,\mathbf{r}_i,...,\mathbf{r}_g\}$, which have opposite directions but the same magnitudes as the supporting forces originated from $\{\mathbf{P}_1^\triangle,...,\mathbf{P}_i^\triangle,...,\mathbf{P}_g^\triangle\}$. Due to the convexity of LBCP, the resultant, or the convex combination of the gravitational force of $\mathcal{I}_i$ and the reaction force $\mathbf{r}_i$ can be produced by $\mathcal{P}_t$. Thus, the stability of items that support $\mathcal{I}_{new}$ remains unaffected.
\end{proof}

Corollary~\ref{corollary:2} suggests that validation is required solely for the stability of the new item, with no need for items already packed. Consequently, this facilitates the integration of the proposed stability validation method with policy learning process due to its computational efficiency.\\

\begin{algorithm}[H]
\caption{Structural Stability Validation}
\label{alg:validate}
\begin{algorithmic}[1] 

\STATE \textbf{Input:} 
\STATE \hspace{0.5cm} New object: $\mathcal{O}_{new}$, Load configuration: $\mathcal{L}_{new}$
\STATE \hspace{0.5cm} Height Map: $\mathcal{HM}_t$, Feasibility Map: $\mathcal{FM}_t$, Tolerance: $\delta_{CoG}$
\STATE \textbf{Output:} Boolean stability flag $valid$

\STATE \textbf{Function} Validate($\mathcal{O}_{new}$, $\mathcal{L}_{new}$, $\mathcal{HM}_t$, $\mathcal{FM}_t$, $\delta_{CoG}$)

\STATE \hspace{0.5cm} Extract object placement coordinates $(x_i, y_i, w_i, d_i)$ from $\mathcal{O}_{new}$

\STATE \hspace{0.5cm} Compute \textbf{support height}:
\STATE \hspace{1.0cm} $h^s \gets \min \mathcal{HM}_t(x_i:x_i+w_i, y_i:y_i+d_i)$

\STATE \hspace{0.5cm} Compute \textbf{contact points at $h^s$}:
\STATE \hspace{1.0cm} $PS_{contact} \gets \{(x, y) \mid \mathcal{HM}_t(x, y) = h^s, x \in [x_i, x_i+w_i], y \in [y_i, y_i+d_i] \}$

\STATE \hspace{0.5cm} Obtain \textbf{points belonging to $\mathcal{P}_t$}:
\STATE \hspace{1.0cm} $PS_{feasible} \gets \{(x, y) \mid \mathcal{FM}_t(x, y) = \text{true}, x \in [x_i, x_i+w_i], y \in [y_i, y_i+d_i] \}$

\STATE \hspace{0.5cm} Compute \textbf{support polygon}:
\STATE \hspace{1.0cm} $\mathbf{P}_{new}^\triangle \gets \mathbf{CH}(PS_{contact} \cap PS_{feasible})$ \quad \textit{(Convex Hull of intersection)}

\STATE \hspace{0.5cm} Compute \textbf{CoG set}:
\STATE \hspace{1.0cm} $\mathcal{C}_{new} \gets$ Compute based on Eq.~\ref{eq:CoG set} using $\delta_{CoG}$

\STATE \hspace{0.5cm} \textbf{Check if CoG is within the support polygon:}
\STATE \hspace{1.0cm} \textbf{if} $\mathcal{C}_{new} \subseteq \mathbf{P}_{new}^\triangle$ \textbf{then}
\STATE \hspace{1.5cm} $valid \gets \text{True}$
\STATE \hspace{1.0cm} \textbf{else}
\STATE \hspace{1.5cm} $valid \gets \text{False}$
\STATE \hspace{1.0cm} \textbf{end if}

\STATE \hspace{0.5cm} \textbf{return} $valid, \mathbf{P}_{new}^\triangle, h^s$

\end{algorithmic}
\end{algorithm}
\subsubsection{Structural Stability Validation}
\label{section:ssv}
To validate if the new item is stably packed in the loading position $\mathcal{L}_{new}$, according to Corollary~\ref{corollary:1}, we need to calculate the support polygon using Eq.~\ref{eq:LBCP}. In addition, due to the uncertainty of the CoG of the item, all extreme points of $\mathcal{C}_{new}$ should be validated.

Note that the time cost increases {\it w.r.t.} the number of LBCPs in Eq.~\ref{eq:LBCP}. To enhance computational efficiency, we create a map called the feasibility map $\mathcal{FM}$, aligning the LBCPs with the heightmap recorded by the camera mounted on the ceiling. 
This approach allows easy identification of the contact point that belongs to $\mathcal{P}_t$ by projecting all points in the contact area into the feasibility map. 
Let $\mathcal{HM}_t(x_i, y_i)$ denote the current heightmap, the distance field from the camera to the bin located at $(x_i, y_i)$. $\mathcal{HM}_t(x_i:x_i+w_i, y_i:y_i+d_i)$ denotes the sliced window. Then, $\mathcal{P}_t$ are projected onto the $xy$-plane creating $\mathcal{FM}$. $\mathcal{FM}_t(x_i, y_i)$ indicates if $(x_i, y_i)$ belongs to an LBCP or not. As every LBCP is mapped to $\mathcal{FM}_t$, verifying a single contact point involves performing only one Boolean operation. Utilizing $\mathcal{HM}_t$ and $\mathcal{FM}_t$, the support polygon $\mathbf{P}_{new}^\triangle$ can be calculated in almost constant time, with minor fluctuations arising from the item's dimensions. The stability validation procedure is detailed in Alg.~\ref{alg:validate}.

\subsubsection{Stability Ensured Packing Policy}
\label{section: stability ensured drl}
We integrate our stability validation method into GOPT \cite{DRL_GOPT}, a state-of-the-art DRL model to address the stability-ensured OBPP. GOPT is an actor-critic DRL model that incorporates two key components: the Placement Generator (PG), which generates a set of Empty Maximal Spaces (EMS) \cite{HEU_EMS} to define feasible placements, and the Packing Transformer, which assembles self-attention and cross attention modules~\cite{Attention}, to serve as the backbone of the DRL framework, enabling the capture of spatial interactions between EMSs and items.

To guarantee stable packing, we introduce the Structural Stability Validation (SSV) module that implements Alg.~\ref{alg:validate}, and the Structural Stability Update (SSU) module that assembles Alg.~\ref{alg:update}. Firstly, a set of placement candidates is extracted with the placement generator. All candidates are validated by the SSV module that relies on the current feasibility map $\mathcal{FM}_t$ to assign a Boolean value to each candidate indicating stability. Meanwhile, all candidates along with the dimensions of the new item are fed into the DRL model. Then, the distribution of actions for packing is derived through the softmax function, which relies on the element-wise product of the GOPT output and the stability mask. After that, the packing action is sampled and executed. Finally, a new LBCP will be appended to $\mathcal{P}_{t+1}$, and SSU module updates $\mathcal{FM}_t$ to $\mathcal{FM}_{t+1}$ using Alg.~\ref{alg:update}.

\begin{algorithm}[H]
\caption{Structural Stability Update}
\label{alg:update}
\begin{algorithmic}[1] 

\STATE \textbf{Input:} 
\STATE \hspace{0.5cm}Feasibility Map: $\mathcal{FM}_t$, Updated LBCPs: $\mathcal{P}_{t+1}$

\STATE \textbf{Output:}$\mathcal{FM}_{t+1}$
\STATE \textbf{Function} UpdateFeasibilityMap($\mathcal{FM}_t$, $\mathcal{P}_{t+1}$)
\STATE \hspace{0.5cm} $\mathcal{FM}_{t+1}\gets{\mathcal{FM}_t}$
\STATE \hspace{0.5cm} Get $\mathbf{P}_{new}^\triangle$ from $\mathcal{P}_{t+1}$
\STATE \hspace{0.5cm} \textbf{for each} $(x,y) \in \mathbf{P}_{new}^\triangle$ \textbf{do}
\STATE \hspace{1.0cm} $\mathcal{FM}_{t+1}(x,y) \gets \text{true}$ \quad \textit{(Mark as feasible)}
\STATE \hspace{0.5cm} \textbf{end for}
\STATE \hspace{0.5cm} \textbf{return} $\mathcal{FM}_{t+1}$

\end{algorithmic}
\end{algorithm}

\subsection{Safe Rearrangement Planning}
\label{section:srp}
This section presents the Stable Rearrangement Planning (SRP) method for re-configuring loading position of previously packed items. The objective is to identify the shortest series of rearrangement actions necessary to make space for a new item or fulfill the bin utilization criteria, all while maintaining the stability of the bin. This task involves three types of operations: Unpacking, Packing, and Repacking.
\begin{itemize}
    \item \textbf{Unpacking} involves removing an item from the bin to the staging area.
    \item \textbf{Packing} refers to loading a new item or an item from staging area into the bin. 
    \item \textbf{Repacking} relocates an item from one position to another inside the bin. 

\end{itemize}
Given the bin state $\mathcal{B}_t$, the LBCPs $\mathcal{P}_t$, the dimensions of the new item $\mathcal{O}_{new}$, and the target bin utilization rate $T_{uti}$, the agent must strategically utilize the three afore-described operations to achieve one of the following objectives:
\begin{itemize}
    \item Pack the new item into the bin, maximizing bin utilization.
    \item Replace a set of items inside the bin with the new item to satisfy a prespecified $T_{uti}$.
\end{itemize}



Specifically, SRP is divided into two primary phases: Search and Path Refinement. Firstly, we use Monte Carlo Tree Search (MCTS)~\cite{MCTS} to find a feasible sequence of rearrangement operations. Then, we minimize the feasible sequence using $A^*$~\cite{A_star}. 

\subsubsection{Monte Carlo Tree Search}
\begin{figure}
    \centering
    \includegraphics[width=0.35\textwidth]{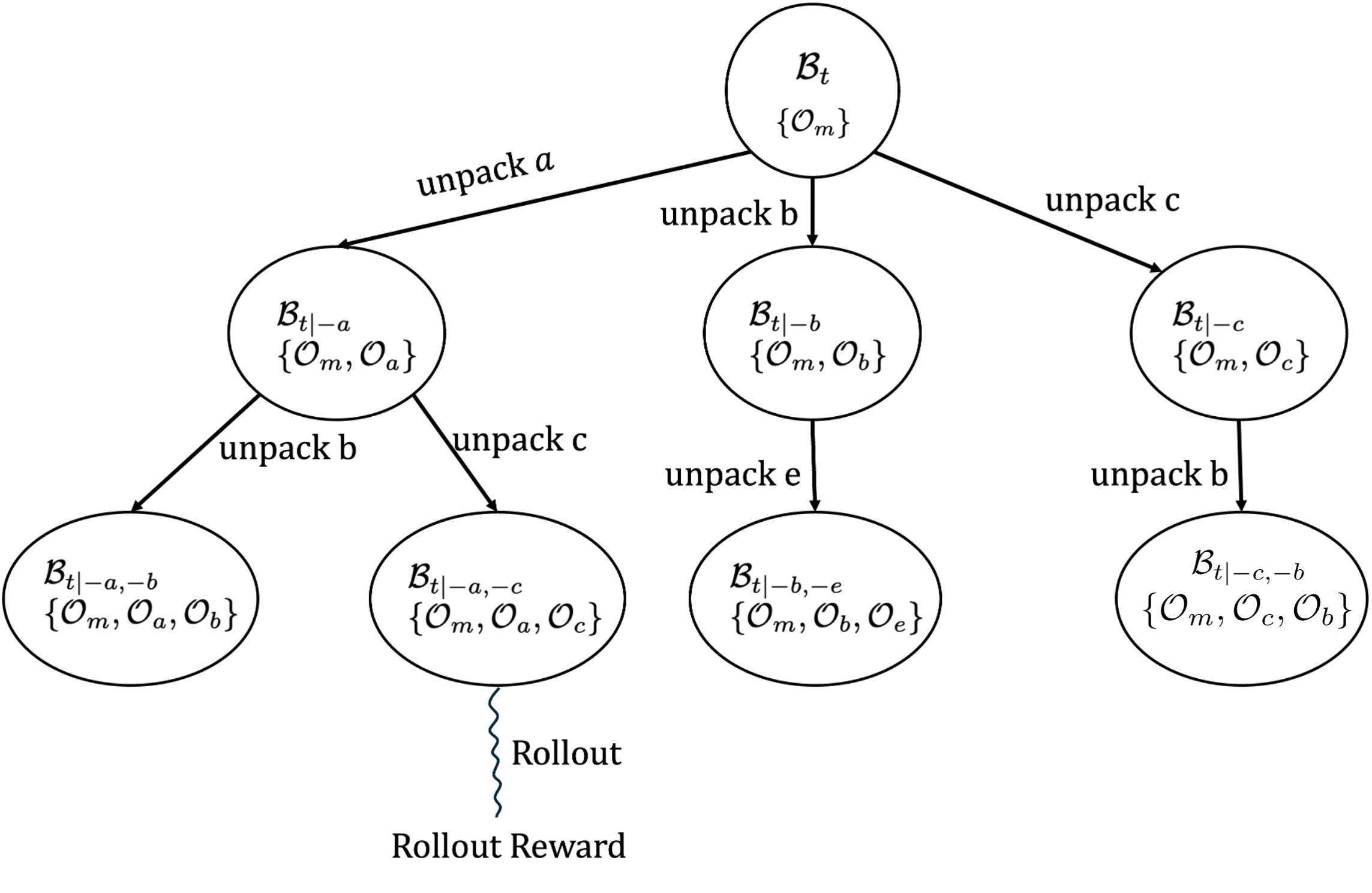}
    \caption{Search tree $\mathcal{T}$: An edge represents an unpacking operation removing an item from the current bin and adding it to the set of unpacked items.}
    \label{fig:mcts}
\end{figure}

MCTS identifies the items that need to be rearranged and determine their new loading positions. The search process is represented by a search tree $\mathcal{T} = (\mathcal{V}, \mathcal{E})$, where $\mathcal{V}$ is the set of nodes and $\mathcal{E}$ is the set of directed edges. 
Fig.~\ref{fig:mcts} illustrates the structure of the tree, where each edge corresponds to a specific operation. To improve search efficiency and prevent repetitive state transitions, we restrict transitions between states to the \textbf{unpacking} action. By structuring the tree in this way, we significantly reduce its complexity compared to a naive approach that evaluates all possible packing, unpacking, and repacking actions at each decision step. As a result, the state $s_i$ of a tree node is characterized by unpacked $\{\mathcal{O}_1, \dots, \mathcal{O}_n\}$ and $\mathcal{B}_{t | -1, \dots, -n}$, {\it i.e.}, $\mathcal{B}_t$ without items $\{\mathcal{O}_1, \dots, \mathcal{O}_n\}$. Given the unpacked items along with newly arrived items, the packing order and their loading positions are determined through a rollout simulation. A feasible rearrangement sequence that meets the specified criteria is deemed a successful outcome if it is identified within the constraint imposed by the maximum allowable number of added nodes. Otherwise, the search process is considered to have failed. Upon a search result, the unpacked items are identified on the basis of the state transitions recorded during the process, while the rollout simulation returns the optimal packing order with associated new loading positions.


Specifically, the search tree $\mathcal{T}$ expands the nodes incrementally based on four steps: selection, expansion, rollout, and backpropagation. In the selection phase, each child node is evaluated using a UCB1 (Upper Confidence Bound) function~\cite{UCB1}, shown in Eq.~\ref{eq:UCB1}, which compromises the exploitation of the learned policy and the exploration for tree expansion.
\begin{equation}
    UCB1(s_i) = \bar{v}_i + \eta\sqrt{\frac{\ln{N}}{n_i}
    }
\label{eq:UCB1}
\end{equation}
where $\bar{v}_i$ is the average rollout reward at state $s_i$, $N$ is the number of visits of its parent, $n_i$ is the number of visits of current node, and $\eta$ is the weight associated with the exploration term set to $\eta=1$.

During the expansion phase, the number of possible unpacking operations corresponds to the number of directly unpackable items. Given a selected node with the bin state $\mathcal{B}_{t|-n, \dots, -m}$, we randomly select one unpackable item $\mathcal{I}_k$, where $k \notin \{n, \dots, m\}$. This item is then added to the set of unpacked items, updating it to $\{ \mathcal{O}_n, \dots, \mathcal{O}_m, \mathcal{O}_k\}$. Consequently, the bin state transitions to $\mathcal{B}_{t|-n, \dots, -m, -k}$.

Since unpacking an item affects the structural stability of the bin, the corresponding LBCP, denoted as $(\mathbf{P}^\triangle_k, h_k^s)$, is removed from the set of LBCPs $\mathcal{P}_{t|-n, \dots, -m}$. As a result, the updated LBCP set becomes $\mathcal{P}_{t|-n, \dots, -m, -k}$.
\begin{equation}
    R_{rollout} = w_v*{Critic(\mathcal{B}_t, \mathcal{O}_{last})} + U_{t}
\label{eq:rollout reward}
\end{equation}

The rollout phase aims to determine the optimal packing sequence for the unpacked items and a new item using the trained DRL model iteratively. In each iteration, packable items are selected via Alg.~\ref{alg:validate}, where $\mathcal{O}_i$ is considered packable if at least one valid loading position exists. The critic network ranks the packable items, and the actor predicts the best placement. The critic approximates the expected return of each state, helping prioritize high-value packing orders and improving value backpropagation through the decision tree.

The rollout reward, defined in Eq.~\ref{eq:rollout reward}, is computed as a weighted sum of two terms: the predicted value of the terminal state, denoted as $Critic(\mathcal{B}_t, \mathcal{O}_{last})$, and the final bin utilization after the rollout simulation, denoted as $U_t$. The weight $w_v$ controls the trade-off between long-term decision quality—reflected by the critic’s prediction—and immediate packing efficiency. The critic's value captures long-term efficiency, while bin utilization reflects immediate packing performance. Together, they guide the agent to optimize both short-term and long-term objectives, aligning with the rearrangement goal.

\subsubsection{$A^*$ for Sequence Refinement}
\begin{figure}
    \centering
    \includegraphics[width=0.3\textwidth]{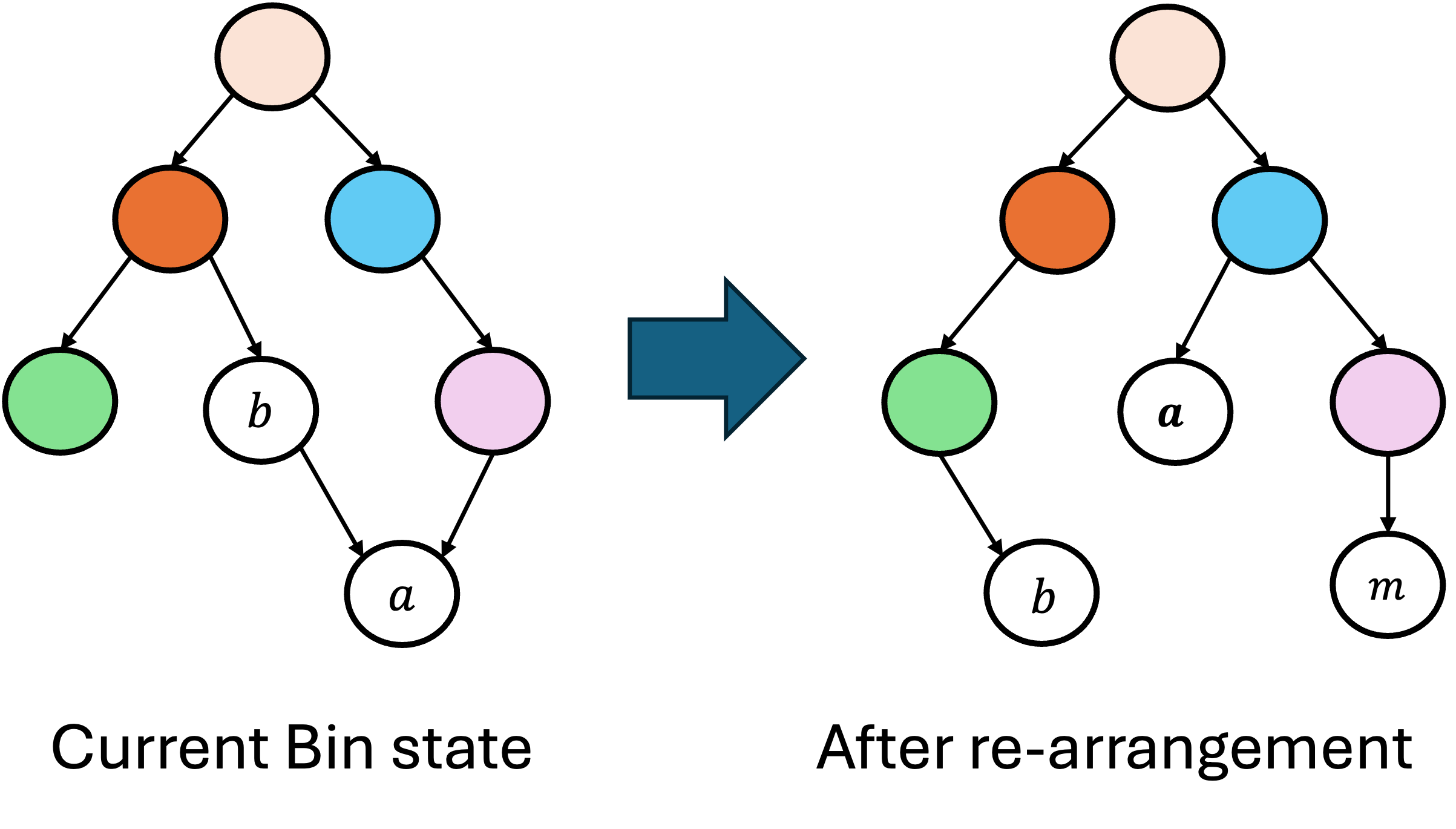}
    \caption{Precedence graphs before and after re-arrangement. Note that $m$ is the new item.}
    \label{fig:precedence}
\end{figure}
\begin{figure}
    \centering
    \includegraphics[width=0.35\textwidth]{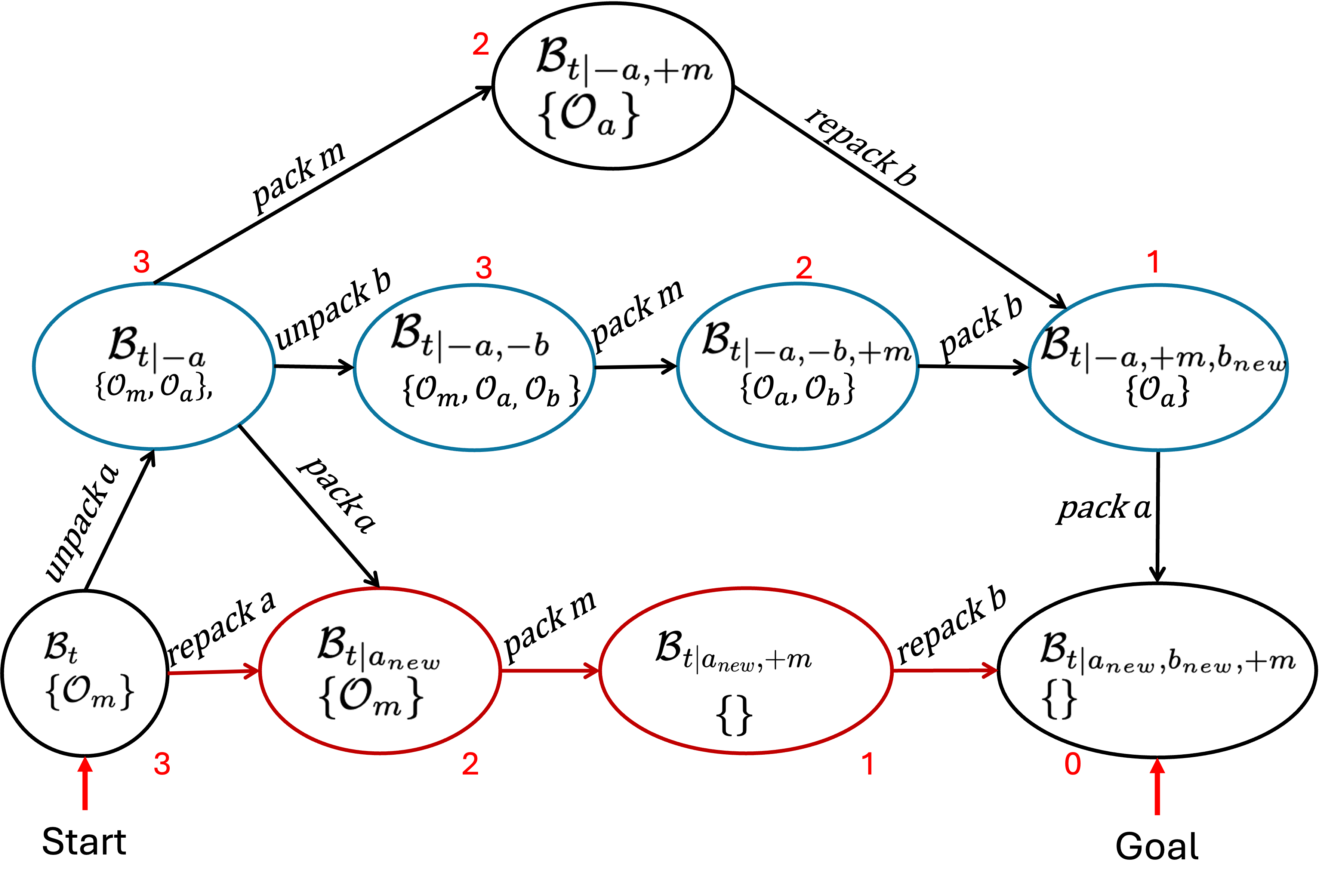}
    \caption{Use of $A^*$ for finding the shortest sequence length of operations. Nodes represent the bin state and the set of unpacked items, while edges represent the operations. The numbers in red are heuristic costs.}
    \label{fig:A*}
\end{figure}

After the rearrangement of sequence via MCTS, the stacking relationships among items may change. We use the precedence graph \cite{DRL_MP_Precedence_xu2023neural}, which encodes the constraints of the pick-up orders, to represent the stacking relationship. Fig.~\ref{fig:precedence} demonstrates an example where the item $b$ cannot be picked up due to the movement block (MB) of item $a$. 

Given the precedence graph before and after the re-arrangement, our goal is to minimize the operations leading the precedence graph before the re-arrangement to the precedence graph after the re-arrangement. This transformation process is modeled as a graph-editing problem, where nodes represent bin state and the set of unpacked items, and edges capture the relationships between them. The transition between nodes are done through packing (adding node), unpacking (deleting node), repacking (rewiring node) operations. Fig.~\ref{fig:A*} shows one example. In accordance with the representation of the bin in MCTS, we denote by $\mathcal{B}_{t|-n, \dots, -m, +k}$ the instance where the new $\mathcal{I}_k$, previously not loaded into the bin, is now packed into it. We denote by $\mathcal{B}_{t|\dots, -m, n_{new}}$ the instance where $\mathcal{I}_n$, previously unpacked from the bin, is now loaded into a new position inside the bin.

To determine the best sequence of operations, the $A^*$ search algorithm methodically examines the realm of potential changes. The search is directed by a heuristic function, characterized by the number of items requiring modification to achieve the target state. The total cost of each node is the summation of the number of operations achieving to the current node and the cost predicted by the heuristic function. Notably, all search nodes are generated dynamically on-the-fly rather than being precomputed. This design ensures memory efficiency and allows adaptive exploration based on current packing states and precedence constraints.

\section{Experiment}
We conducted the experiment validating the efficiency of the proposed stability validation method. The performance of the trained DRL model is then demonstrated in the \textbf{RS Dataset}\cite{DRL_Dataset_zhao2021online}. Subsequently, we exhibit how the proposed SRP method improves bin utilization and its effectiveness when compared to the baseline. Finally, we show a real-world testing case study.

\subsection{Structural Stability}
\begin{figure*}
    \centering
    \includegraphics[width=0.65\textwidth]{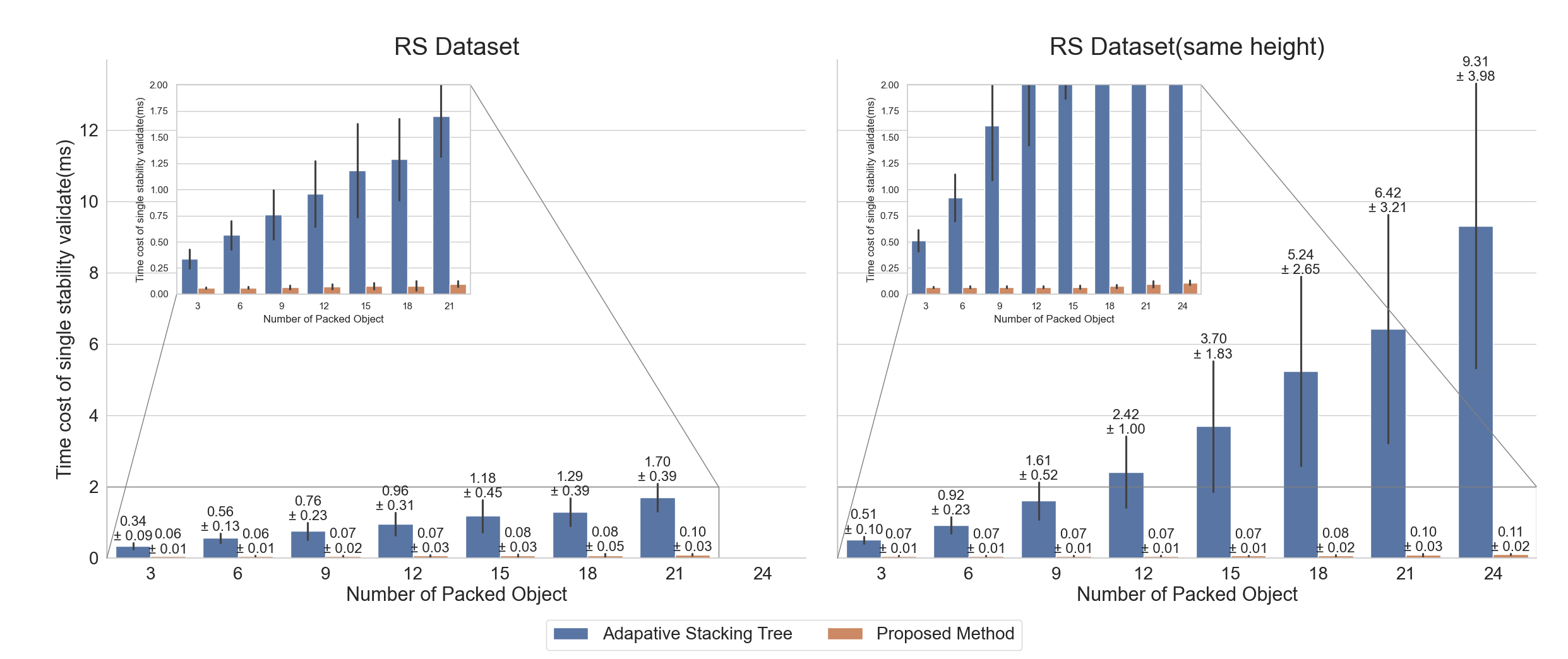}
    \caption{Time cost for completing the validation of all pre-specified loading positions for the new item. Each bar is annotated with the mean and standard deviation in milliseconds.}
    \label{fig:single packing time cost}
\end{figure*}
The baseline for validating the stability is the Adaptive Stacking Tree \cite{Stability_zhao2022learning}, since it attains the state-of-the-art accuracy while improving computational efficiency. Specifically, it is a data structure designed to efficiently update the mass distribution, and hence it assumes that the mass of the item is known in advance. The tree is updated in a top-down fashion, which means that stability updates propagate downward in an efficient manner. 

Note that any loading operation verified as stable by our method can ensure the bin structural stability, as evidenced by a random stable loading experiment on a real platform\footnote{\url{https://drive.google.com/file/d/1N7vhnvhIHXhAKGoktYj1EzQf3IT7t46c/view?usp=share_link}}. Therefore, we only assess the efficiency of the proposed approach. The data set is derived from the \textbf{RS dataset} along with a modified version referred to as \textbf{RS dataset (same height)}. In the \textbf{RS dataset}, each sequence of items is selected without replacement from a collection of 64 items with varying sizes. In contrast, in the \textbf{RS dataset (same height)}, the sequences are sampled without replacement from a subset of the \textbf{RS dataset}, where all items have the same height of $0.12$ meters.

We generated $1,000$ sequences of items and each sequence has length $500$, and we adopt a random packing policy that randomly selects a stable loading position from a set of pre-defined positions. Given a sequence of items, we firstly validate the set of loading positions for the upcoming item. If no stable loading positions are available, this item will be simply ignored and validate the stable loading position for the next item. If stable loading positions are available, a random policy simply selects one loading position for packing the item. This process terminates when the whole sequence is traversed. When packing each sequence of items, we record the time cost in every three packing iterations shown in Fig.~\ref{fig:single packing time cost}. 

In general, the proposed method is at least $5$ times faster than the baseline method. With the increase of the number of packed items, the time cost becomes higher for the baseline method, while the time cost of our proposed method remains nearly constant. In addition, the efficiencies of the baseline method for dealing with different dataset are significantly different. The time cost becomes even higher in the \textbf{RS dataset (same height)}. The reason is that the packed item potentially has much more parents so that it requires much more update on the edge of the adaptive stacking tree. On the other hand, our method is not affected by the variation of the dataset.

\subsection{DRL Model Performance}
We train and evaluate the GOPT model on the \textbf{RS dataset} using the Tianshou framework~\cite{Tianshou}, adopting the same hyperparameter settings as in~\cite{DRL_GOPT}. Unlike~\cite{DRL_GOPT}, our training phase incorporates the static stability module (Section~\ref{section:ssv}) to filter out unstable loading position candidates, ensuring that only stable placements are used to train the network.

During the testing phase, we assess not only the performance of the actor network but also that of the critic network. This is motivated by the fact that the critic network plays a crucial role in guiding the search tree expansion within the SRP module (Section~\ref{section:srp}). To evaluate the critic network, we assume access to multiple lookahead items and utilize the critic to rank them. Items with higher predicted critic values are prioritized for placement, under the hypothesis that doing so results in higher overall bin utilization.

We generated $2,000$ sequences of items for testing. Fig.~\ref{fig:DRL performance} reports the model performance {\it w.r.t.} different number of items of the lookahead. Fig.~\ref{fig:DRL performance} shows that the bin utilization increases monotonously, with increasing number of item forecast. The observation also gives strong evidence that justifies the use of the critic network to find the feasible sequence in MCTS. 
\begin{figure}
    \centering
    \includegraphics[width=0.35\textwidth]{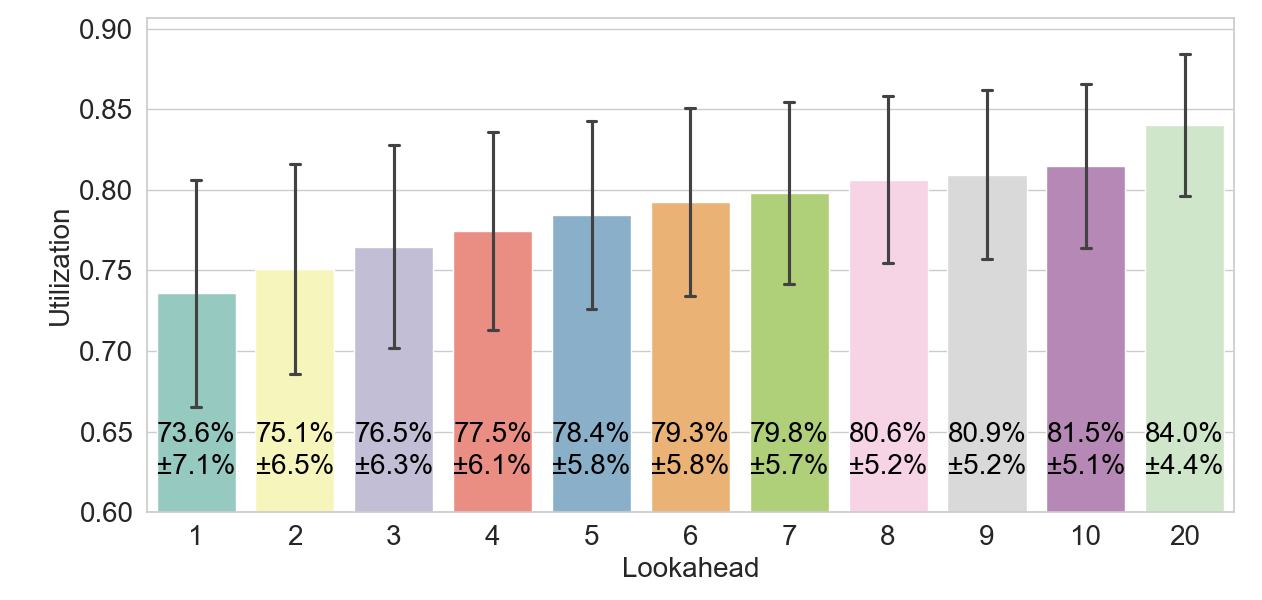}
    \caption{Performance of the trained DRL model on the \textbf{RS} test set. The horizontal axis represents the number of lookahead items, and the vertical axis represents the bin utilization.}
    \label{fig:DRL performance}
\end{figure}
\subsection{Stable Rearrangement Planning}
\begin{figure}
    \centering
    \includegraphics[width=0.35\textwidth]{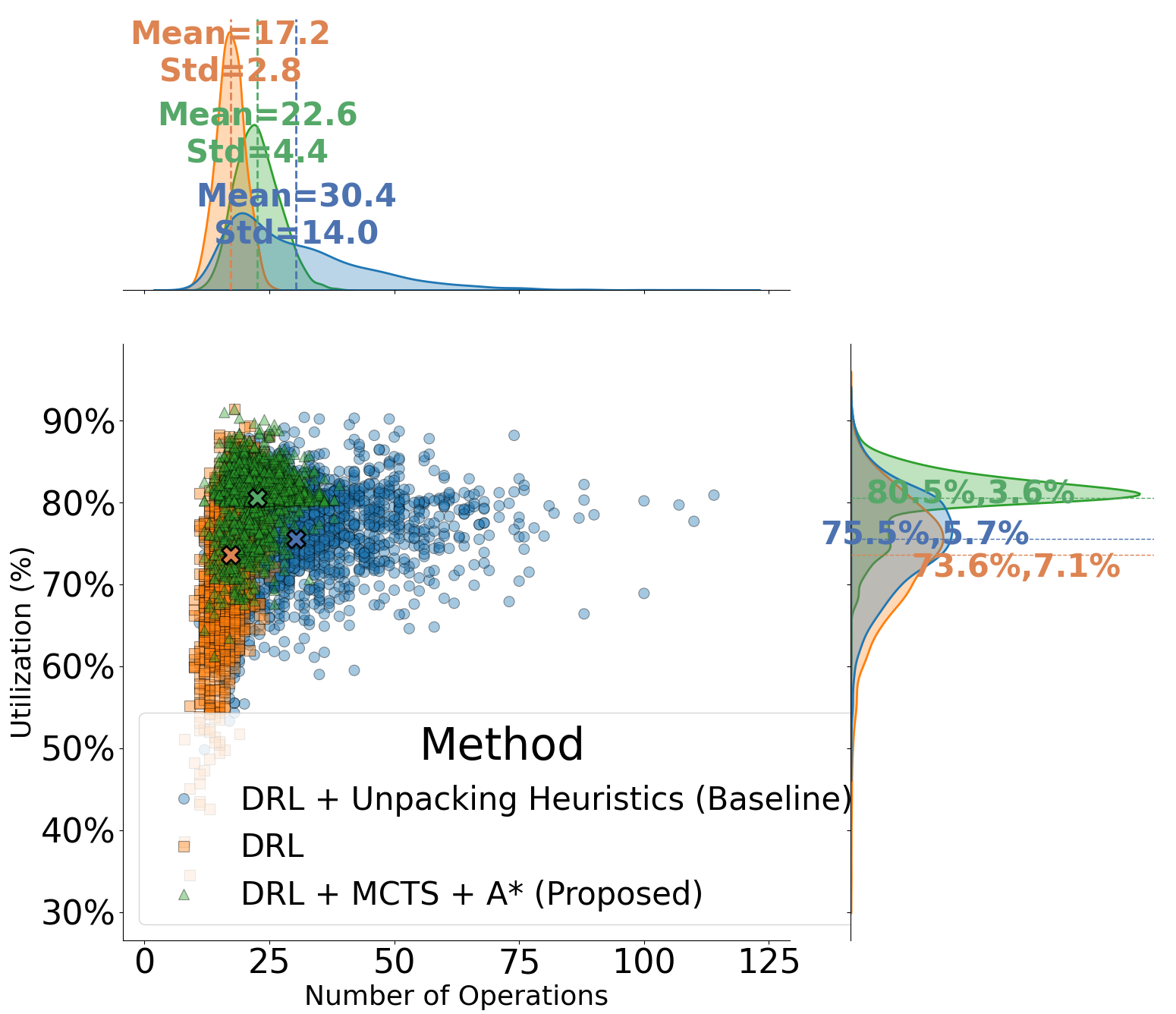}
    \caption{Joint distribution of the number of operations and bin utilization.}
    \label{fig:re-arrangment result}
\end{figure}
In this experiment, we aim to verify the following.
\begin{itemize}
    \item The contribution on improving the bin utilization.
    \item The efficiency of the proposed method compared with unpacking heuristics~\cite{DRL_MP_yang2023heuristics}.
    \item The success rate of the MCTS.
    \item The contribution of the use of $A^*$ for minimizing the number of operations.
\end{itemize}
We utilize the same test set to evaluate the improvement in bin utilization. For each sequence of items, the rearrangement is triggered if the new item cannot be packed into the bin. In the rearrangement phase, we run the MCTS algorithm to find the feasible sequence, and $A^*$ to shorten the sequence. In each MCTS run, the maximum number of nodes is set to $100$. The weight $w_v$we choose for calculating rollout reward is $5$. To control the branching factor, each node can have $3$ child nodes at maximum. To compromise the time cost for finding the solution and execution of the operation, we set the maximum search depth at $6$, which means that the maximum number of unpacked items is $6$. The search is deemed fail if no feasible sequence was found after expanding the tree to have $100$ nodes. We set the $T_{uti}$ to $0.8$. If the bin utilization is above the threshold, the search terminates even though some items remains not to be packed.

For comparison purposes, we implement the baseline~\cite{DRL_MP_yang2023heuristics} that employs a heuristic to select unpacking item. Given a new item, the unpacking heuristic evaluates the items laid on the top layer of the bin via a linear function that compromises the increase of the bin utilization by exchanging this item with the new item, and the saved wasted space by unpacking the item. We use the same weights as~\cite{DRL_MP_yang2023heuristics}. To control the number of unpacking operations and for fair comparison, we set the maximum staging capacity to $6$. 

Fig.~\ref{fig:re-arrangment result} shows the joint distribution between the number of operations and the bin utilization. In Fig.~\ref{fig:re-arrangment result}, to show the improvement of the bin utilization, we also show the result of DRL model without re-arrangement. Based on the result, bin utilization was improved to $80.5\%$ using the proposed rearrangement method, while bin utilization was improved to $75.5\%$ from $73.6\%$ using the unpacking heuristics. Therefore, our method attains a better performance in terms of the bin utilization. On the other hand, both the proposed method and the baseline method increase the number of operations due to re-arrangement operations. Our method requires approximately $8$ fewer operations than the baseline approach.

In addition, the overall success rate of finding the feasible sequence using MCTS is $83.9\%$, and Fig.~\ref{fig:mcts success rate} shows the success rate {\it w.r.t.} different bin states which are categorized by the bin utilization. This figure reveals that with the increase of the bin utilization, the success rate of MCTS decreases gradually. It indicates that more searching iterations are needed with the increase of the bin utilization. 
\begin{figure}
    \centering
    \includegraphics[width=0.48\textwidth]{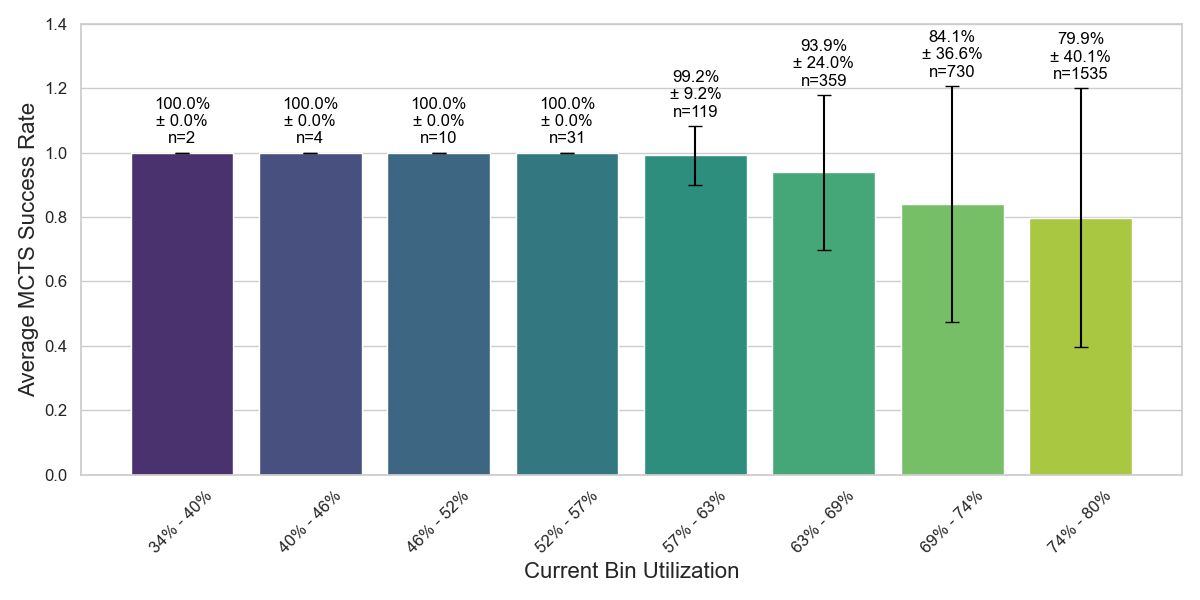}
    \caption{Success rate of the MCTS for finding the feasible sequence {\it w.r.t.} different bin states which are categorized by the bin utilization. Each bin is annotated with the mean success rate and standard deviation.}
    \label{fig:mcts success rate}
\end{figure}

To demonstrate the advantage of using $A^*$ in reducing the operation sequence length, we compare the sequence lengths identified by MCTS prior to and following the implementation of $A^*$. Our results indicate that $A^*$ shortens the sequence length by approximately $31\%$. Specifically, MCTS identifies sequences with an average length of $5.8$ and a standard deviation of $2.1$, whereas the sequences refined by $A^*$ have an average length of $4.0$ with a standard deviation of $1.6$.


\subsection{Real Experiment}
\begin{figure*}
    \centering
    \includegraphics[width=0.95\textwidth]{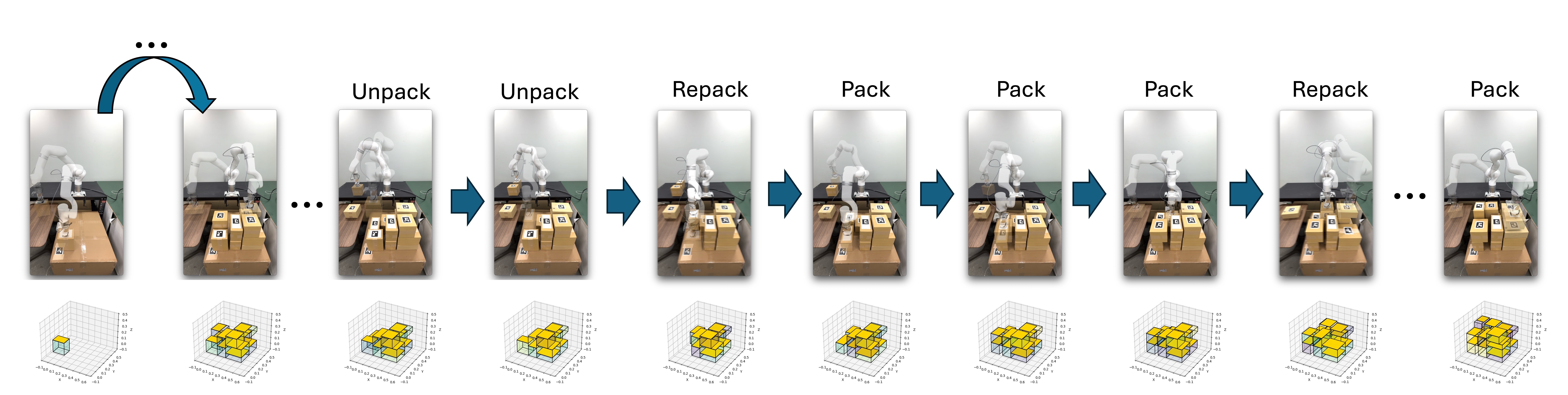}
    \caption{Real packing test with cardboard boxes of various sizes. The first two snapshots illustrate the sequential loading of new items, while the subsequent snapshots demonstrate our SRP approach creating space for the new item, with the corresponding LBCPs.}
    \label{fig:real_snapshots}
\end{figure*}

This section describes the deployment of our packing system on a physical platform. We used a 6-DOF robot arm with an RGB-D sensor mounted on its end-effector. The items are custom-made cardboard boxes in six distinct sizes. Each box is labeled with an Aruco marker~\cite{garrido2014automatic}, enabling the robot to estimate its pose and dimensions relative to the base frame.

The system observes one box at a time and places it into a pallet of size $(55\,\text{cm} \times 45\,\text{cm} \times 45\,\text{cm})$. When the pallet cannot accommodate a new item, the arm executes one of the following actions:  
1) remove a packed box and move it to a staging area (max 4 boxes),  
2) place a new or staged box into the pallet, or  
3) reposition a packed box without removing it.

In real-world scenarios, preventing collisions during packing is a critical concern. Without adjustments to the simulation, new items often collide with previously packed ones. Following~\cite{DRL_GOPT}, we add buffer space to each packed item by virtually enlarging its dimensions—while keeping the physical size unchanged. However, this may lead to unstable placements, as the contact area between the new item and enlarged LBCPs can be misestimated.

We evaluated the reliability of each loading position by checking the consistency of contact points under slight variations in item size. Specifically, a sliced window is extracted from the heightmap based on the item’s dimensions, then shrunk by a fixed offset to simulate a smaller item. If the maximum height points (i.e., contact points) within the original and shrunk windows differ in height, it indicates high sensitivity to size variation and potential instability. Such positions are discarded to enhance packing robustness. For stability validation, the support polygon is now computed using contact points from the shrunk window and $\mathcal{FM}_t$, mitigating errors caused by enlarged LBCPs.

Fig.~\ref{fig:real_snapshots} illustrates a robot-assisted packing scenario. The packing process terminates when our SRP cannot find the operation sequence to accommodate the new item. Finally, our packing system successfully packed $16$ items into the bin using $24$ operation steps, achieving the bin utilization of $71.2\%$. Taking into account the practical limitations of avoiding collisions and out-of-distribution item sizes, this outcome remains comparable to the result achieved in simulation ($80.5\%$).

\section{Conclusion}
This work proposed a novel framework for tackling the Online 3D Bin Packing Problem (OBPP), emphasizing both stability and real-time space utilization. We introduced the concept of Load-Bearable Convex Polygon (LBCP), a lightweight yet robust stability validation method that avoids reliance on exact mass distributions. This enables fast, stable decisions suitable for uncertain industrial environments.

By integrating this validation into a state-of-the-art DRL model, we ensured that only stable placements were considered during policy learning and execution. The approach achieved competitive bin utilization while satisfying safety constraints. Additionally, the Stable Rearrangement Planning (SRP) module enabled efficient reconfiguration of packed items without full unpacking, reducing operational steps.

Our experiments validated the computational efficiency and DRL compatibility of the stability module, as well as the SRP’s effectiveness in minimizing action sequences. The successful deployment on a real robotic system further demonstrates the method’s practical viability and scalability. Future work will explore more complex item geometries, dynamic task constraints, and multi-agent coordination for large-scale automation.


\bibliographystyle{IEEEtran}
\bibliography{biblio}

\begin{thebibliography}{10}
\providecommand{\url}[1]{#1}
\csname url@samestyle\endcsname
\providecommand{\newblock}{\relax}
\providecommand{\bibinfo}[2]{#2}
\providecommand{\BIBentrySTDinterwordspacing}{\spaceskip=0pt\relax}
\providecommand{\BIBentryALTinterwordstretchfactor}{4}
\providecommand{\BIBentryALTinterwordspacing}{\spaceskip=\fontdimen2\font plus
\BIBentryALTinterwordstretchfactor\fontdimen3\font minus \fontdimen4\font\relax}
\providecommand{\BIBforeignlanguage}[2]{{%
\expandafter\ifx\csname l@#1\endcsname\relax
\typeout{** WARNING: IEEEtran.bst: No hyphenation pattern has been}%
\typeout{** loaded for the language `#1'. Using the pattern for}%
\typeout{** the default language instead.}%
\else
\language=\csname l@#1\endcsname
\fi
#2}}
\providecommand{\BIBdecl}{\relax}
\BIBdecl

\bibitem{HEU_christensen2016multidimensional}
H.~I. Christensen, A.~Khan, S.~Pokutta, and P.~Tetali, ``Multidimensional bin packing and other related problems: A survey,'' \emph{Computer Science Review}, 2016.

\bibitem{DRL_Dataset_zhao2021online}
H.~Zhao, Q.~She, C.~Zhu, Y.~Yang, and K.~Xu, ``Online 3d bin packing with constrained deep reinforcement learning,'' in \emph{Proceedings of the AAAI Conference on Artificial Intelligence}, vol.~35, no.~1, 2021, pp. 741--749.

\bibitem{DRL_PCT}
H.~Zhao, Y.~Yu, and K.~Xu, ``Learning efficient online 3d bin packing on packing configuration trees,'' in \emph{International conference on learning representations}, 2021.

\bibitem{DRL_2021_yang2021packerbot}
Z.~Yang, S.~Yang, S.~Song, W.~Zhang, R.~Song, J.~Cheng, and Y.~Li, ``Packerbot: Variable-sized product packing with heuristic deep reinforcement learning,'' in \emph{2021 IEEE/RSJ International Conference on Intelligent Robots and Systems (IROS)}.\hskip 1em plus 0.5em minus 0.4em\relax IEEE, 2021, pp. 5002--5008.

\bibitem{Stability_zhao2022learning}
H.~Zhao, C.~Zhu, X.~Xu, H.~Huang, and K.~Xu, ``Learning practically feasible policies for online 3d bin packing,'' \emph{Science China Information Sciences}, vol.~65, no.~1, p. 112105, 2022.

\bibitem{DRL_GOPT}
H.~Xiong, C.~Guo, J.~Peng, K.~Ding, W.~Chen, X.~Qiu, L.~Bai, and J.~Xu, ``Gopt: Generalizable online 3d bin packing via transformer-based deep reinforcement learning,'' \emph{IEEE Robotics and Automation Letters}, vol.~9, no.~11, pp. 10\,335--10\,342, 2024.

\bibitem{DRL_wu2024efficient}
Z.~Wu, Y.~Li, W.~Zhan, C.~Liu, Y.-H. Liu, and M.~Tomizuka, ``Efficient reinforcement learning of task planners for robotic palletization through iterative action masking learning,'' \emph{IEEE Robotics and Automation Letters}, 2024.

\bibitem{DRL_kang2024gradual}
M.~Kang, H.~Kee, Y.~Park, J.~Kim, J.~Jeong, G.~Cheon, J.~Lee, and S.~Oh, ``Gradual receptive expansion using vision transformer for online 3d bin packing,'' in \emph{2024 IEEE/RSJ International Conference on Intelligent Robots and Systems (IROS)}.\hskip 1em plus 0.5em minus 0.4em\relax IEEE, 2024, pp. 7338--7343.

\bibitem{DRL_mu20253d}
X.~Mu, Q.~Kan, Y.~Jiang, C.~Chang, X.~Tian, L.~Zhou, and Y.~Zhao, ``3d vision robot online packing platform for deep reinforcement learning,'' \emph{Robotics and Computer-Integrated Manufacturing}, vol.~94, p. 102941, 2025.

\bibitem{DRL_Precedence_hu2020tap}
R.~Hu, J.~Xu, B.~Chen, M.~Gong, H.~Zhang, and H.~Huang, ``Tap-net: transport-and-pack using reinforcement learning,'' \emph{ACM Transactions on Graphics (TOG)}, vol.~39, no.~6, pp. 1--15, 2020.

\bibitem{DRL_MP_Precedence_xu2023neural}
J.~Xu, M.~Gong, H.~Zhang, H.~Huang, and R.~Hu, ``Neural packing: from visual sensing to reinforcement learning,'' \emph{ACM Transactions on Graphics (TOG)}, vol.~42, no.~6, pp. 1--11, 2023.

\bibitem{DRL_unpack_song2023towards}
S.~Song, S.~Yang, R.~Song, S.~Chu, W.~Zhang \emph{et~al.}, ``Towards online 3d bin packing: Learning synergies between packing and unpacking via drl,'' in \emph{Conference on Robot Learning}.\hskip 1em plus 0.5em minus 0.4em\relax PMLR, 2023, pp. 1136--1145.

\bibitem{DRL_MP_yang2023heuristics}
S.~Yang, S.~Song, S.~Chu, R.~Song, J.~Cheng, Y.~Li, and W.~Zhang, ``Heuristics integrated deep reinforcement learning for online 3d bin packing,'' \emph{IEEE Transactions on Automation Science and Engineering}, vol.~21, no.~1, pp. 939--950, 2023.

\bibitem{Offline_DRL_zhang2021attend2pack}
J.~Zhang, B.~Zi, and X.~Ge, ``Attend2pack: Bin packing through deep reinforcement learning with attention,'' \emph{arXiv preprint arXiv:2107.04333}, 2021.

\bibitem{Offline_DRL_zhu2021learning}
Q.~Zhu, X.~Li, Z.~Zhang, Z.~Luo, X.~Tong, M.~Yuan, and J.~Zeng, ``Learning to pack: A data-driven tree search algorithm for large-scale 3d bin packing problem,'' in \emph{Proceedings of the 30th ACM International Conference on Information \& Knowledge Management}, 2021, pp. 4393--4402.

\bibitem{Offline_DRL_jiang2021learning}
Y.~Jiang, Z.~Cao, and J.~Zhang, ``Learning to solve 3-d bin packing problem via deep reinforcement learning and constraint programming,'' \emph{IEEE transactions on cybernetics}, vol.~53, no.~5, pp. 2864--2875, 2021.

\bibitem{Offline_DRL_que2023solving}
Q.~Que, F.~Yang, and D.~Zhang, ``Solving 3d packing problem using transformer network and reinforcement learning,'' \emph{Expert Systems with Applications}, vol. 214, p. 119153, 2023.

\bibitem{Offline_DRL_pan2024ppn}
J.-H. Pan, X.~Gao, K.-H. Hui, S.~Zhu, Y.-H. Liu, P.-A. Heng, and C.-W. Fu, ``Ppn-pack: Placement proposal network for efficient robotic bin packing,'' \emph{IEEE Robotics and Automation Letters}, 2024.

\bibitem{Offline_DRL_packing_order}
B.~Wang, Z.~Lin, W.~Kong, and H.~Dong, ``Bin packing optimization via deep reinforcement learning,'' \emph{IEEE Robotics and Automation Letters}, vol.~10, no.~3, pp. 2542--2549, 2025.

\bibitem{Stability_Wang_8794049}
F.~Wang and K.~Hauser, ``Stable bin packing of non-convex 3d objects with a robot manipulator,'' in \emph{2019 International Conference on Robotics and Automation (ICRA)}, 2019, pp. 8698--8704.

\bibitem{Stability_10556686}
R.~Liu, K.~Deng, Z.~Wang, and C.~Liu, ``Stablelego: Stability analysis of block stacking assembly,'' \emph{IEEE Robotics and Automation Letters}, vol.~9, no.~11, pp. 9383--9390, 2024.

\bibitem{MCTS}
R.~Coulom, ``Efficient selectivity and backup operators in monte-carlo tree search,'' in \emph{International conference on computers and games}.\hskip 1em plus 0.5em minus 0.4em\relax Springer, 2006, pp. 72--83.

\bibitem{A_star}
\BIBentryALTinterwordspacing
P.~Hart, N.~Nilsson, and B.~Raphael, ``A formal basis for the heuristic determination of minimum cost paths,'' \emph{{IEEE} Transactions on Systems Science and Cybernetics}, vol.~4, no.~2, pp. 100--107, 1968. [Online]. Available: \url{https://doi.org/10.1109/tssc.1968.300136}
\BIBentrySTDinterwordspacing

\bibitem{CNN}
\BIBentryALTinterwordspacing
K.~O'Shea and R.~Nash, ``An introduction to convolutional neural networks,'' 2015. [Online]. Available: \url{https://arxiv.org/abs/1511.08458}
\BIBentrySTDinterwordspacing

\bibitem{HEU_EMS}
F.~Parre{\~n}o, R.~Alvarez-Vald{\'e}s, J.~M. Tamarit, and J.~F. Oliveira, ``A maximal-space algorithm for the container loading problem,'' \emph{INFORMS Journal on Computing}, vol.~20, no.~3, pp. 412--422, 2008.

\bibitem{Attention}
\BIBentryALTinterwordspacing
A.~Vaswani, N.~Shazeer, N.~Parmar, J.~Uszkoreit, L.~Jones, A.~N. Gomez, L.~u. Kaiser, and I.~Polosukhin, ``Attention is all you need,'' in \emph{Advances in Neural Information Processing Systems}, I.~Guyon, U.~V. Luxburg, S.~Bengio, H.~Wallach, R.~Fergus, S.~Vishwanathan, and R.~Garnett, Eds., vol.~30.\hskip 1em plus 0.5em minus 0.4em\relax Curran Associates, Inc., 2017. [Online]. Available: \url{https://proceedings.neurips.cc/paper_files/paper/2017/file/3f5ee243547dee91fbd053c1c4a845aa-Paper.pdf}
\BIBentrySTDinterwordspacing

\bibitem{Stability_ramos2016physical}
A.~G. Ramos, J.~F. Oliveira, and M.~P. Lopes, ``A physical packing sequence algorithm for the container loading problem with static mechanical equilibrium conditions,'' \emph{International Transactions in Operational Research}, vol.~23, no. 1-2, pp. 215--238, 2016.

\bibitem{Stability_gzara2020pallet}
F.~Gzara, S.~Elhedhli, and B.~C. Yildiz, ``The pallet loading problem: Three-dimensional bin packing with practical constraints,'' \emph{European Journal of Operational Research}, vol. 287, no.~3, pp. 1062--1074, 2020.

\bibitem{Stability_ZHU2024109814}
\BIBentryALTinterwordspacing
W.~Zhu, S.~Chen, M.~Dai, and J.~Tao, ``Solving a 3d bin packing problem with stacking constraints,'' \emph{Computers \& Industrial Engineering}, vol. 188, p. 109814, 2024. [Online]. Available: \url{https://www.sciencedirect.com/science/article/pii/S0360835223008380}
\BIBentrySTDinterwordspacing

\bibitem{Stability_ali2025static}
S.~Ali, A.~G. Ramos, and J.~F. Oliveira, ``Static stability versus packing efficiency in online three-dimensional packing problems: A new approach and a computational study,'' \emph{Computers \& Operations Research}, p. 107005, 2025.

\bibitem{Stability_peiwen}
P.~Zhou, Z.~Gao, C.~Li, and N.~Y. Chong, ``An efficient deep reinforcement learning model for online 3d bin packing combining object rearrangement and stable placement,'' in \emph{2024 24th International Conference on Control, Automation and Systems (ICCAS)}, 2024, pp. 964--969.

\bibitem{Stability_10631682}
Z.~Wu, Y.~Li, W.~Zhan, C.~Liu, Y.-H. Liu, and M.~Tomizuka, ``Efficient reinforcement learning of task planners for robotic palletization through iterative action masking learning,'' \emph{IEEE Robotics and Automation Letters}, vol.~9, no.~11, pp. 9303--9310, 2024.

\bibitem{Stability_zhang2025physics}
T.~Zhang, Z.~Wu, Y.~Chen, Y.~Wang, B.~Liang, S.~Moura, M.~Tomizuka, M.~Ding, and W.~Zhan, ``Physics-aware robotic palletization with online masking inference,'' \emph{arXiv preprint arXiv:2502.13443}, 2025.

\bibitem{Stability_MAZUR2025100329}
\BIBentryALTinterwordspacing
P.~G. Mazur, J.~W. Melsbach, and D.~Schoder, ``Physical question, virtual answer: Optimized real-time physical simulations and physics-informed learning approaches for cargo loading stability,'' \emph{Operations Research Perspectives}, vol.~14, p. 100329, 2025. [Online]. Available: \url{https://www.sciencedirect.com/science/article/pii/S2214716025000053}
\BIBentrySTDinterwordspacing

\bibitem{UCB1}
\BIBentryALTinterwordspacing
P.~Auer, N.~Cesa-Bianchi, and P.~Fischer, ``Finite-time analysis of the multiarmed bandit problem,'' \emph{Machine Learning}, vol.~47, no. 2-3, pp. 235--256, 2002. [Online]. Available: \url{https://link.springer.com/article/10.1023/A:1013689704352}
\BIBentrySTDinterwordspacing

\bibitem{Tianshou}
J.~Weng, H.~Chen, D.~Yan, K.~You, A.~Duburcq, M.~Zhang, Y.~Su, H.~Su, and J.~Zhu, ``Tianshou: A highly modularized deep reinforcement learning library,'' \emph{Journal of Machine Learning Research}, vol.~23, no. 267, pp. 1--6, 2022.

\bibitem{garrido2014automatic}
S.~Garrido-Jurado, R.~Mu{\~n}oz-Salinas, F.~J. Madrid-Cuevas, and M.~J. Mar{\'\i}n-Jim{\'e}nez, ``Automatic generation and detection of highly reliable fiducial markers under occlusion,'' \emph{Pattern Recognition}, vol.~47, no.~6, pp. 2280--2292, 2014.

\end{thebibliography}

\end{document}